\documentclass[12pt]{article}
\usepackage{amsmath, amssymb, amsthm}

\usepackage[numbers]{natbib} 

\usepackage{graphicx}
\usepackage{algorithm, algpseudocode}

\newtheorem{theorem}{Theorem}

\newtheorem{proposition}{Proposition}

\newtheorem{example}{Example}
\newtheorem{assumption}{Assumption}

\title{Incorporating structural uncertainty in causal decision making.}

\author{Maurits Kaptein\thanks{Corresponding author. Email: m.c.kaptein@tue.nl} \\ 
        \small Department of Mathematics and Computer Science. \\
        \small University of Eindhoven. \\
        \small De Zaale 1, 5600MB Eindhoven, \\
        \small  the Netherlands. \\[0.5em]}
        
\date{}

\begin{document}

\maketitle

\begin{abstract}
Practitioners making decisions based on causal effects typically ignore structural uncertainty. We analyze when this uncertainty is consequential enough to warrant methodological solutions (Bayesian model averaging over competing causal structures). Focusing on bivariate relationships ($X \rightarrow Y$ vs. $X \leftarrow Y$), we establish that model averaging is beneficial when: (1) structural uncertainty is moderate to high, (2) causal effects differ substantially between structures, and (3) loss functions are sufficiently sensitive to the size of the causal effect. We prove optimality results of our suggested methodological solution under regularity conditions and demonstrate through simulations that modern causal discovery methods can provide, within limits, the necessary quantification. Our framework complements existing robust causal inference approaches by addressing a distinct source of uncertainty typically overlooked in practice.
\end{abstract}



\section{Introduction}

Evidence-based decision-making has become the gold standard across numerous domains, from public policy to business strategy. The typical approach involves estimating causal effects within an assumed causal structure and then using these estimates to guide decisions. However, this conventional wisdom overlooks a fundamental source of uncertainty: uncertainty about the causal relationships themselves.

Consider a policymaker evaluating whether to expand a job training program. Standard practice would involve estimating the causal effect of training on employment outcomes and using this estimate to decide on program expansion. This approach often implicitly assumes that the causal relationship flows from training to employment. But what if employment prospects actually influence training participation? The direction of the causal arrow potentially fundamentally alters the decision calculus, yet practitioners rarely account for uncertainty regarding the causal arrow(s) in a given decision problem.

This scenario exemplifies a broader challenge in causal inference. While sophisticated methodologies exist for handling parameter uncertainty within assumed causal structures \citep{imbens2015causal, hahn2020bayesian}, uncertainty about the structures themselves receives considerably less attention in practical decision making. Although the causal discovery literature has developed powerful tools for learning causal structures from data \citep{spirtes2000causation, glymour2016review}, these methods typically yield point estimates rather than principled measures of structural uncertainty. This gap between causal discovery and decision-making, in our view, represents a missed opportunity to improve decision quality in settings where structural uncertainty is substantial.

We partly address this gap by examining a Bayesian model averaging framework for incorporating structural uncertainty into decision-making. We examine this framework both theoretically and practically; in both cases focusing primarily on the bivariate case (i.e., $X \rightarrow Y$ vs. $X \leftarrow Y$) case. Our theoretical analysis reveals that the value of accounting for structural uncertainty depends mainly on three key factors. First, the magnitude of structural uncertainty itself matters—extreme certainty in either direction reduces the benefits of averaging. Second, the extent to which different causal structures imply different optimal actions determines the stakes involved. Third, the sensitivity of the loss function to action differences amplifies or diminishes the importance of structural uncertainty. These factors interact in complex ways that our framework helps practitioners navigate.

We make a number of contributions to the literature. First, we characterize when structural uncertainty is decision-relevant by analyzing the interaction between causal effect differences and loss function sensitivity. This characterization provides practical guidance on when to invest in structural uncertainty quantification. Second, we establish the optimality of our proposed model averaging framework under regularity conditions. Third, we demonstrate, using simulations, how modern bivariate causal discovery methods can provide structural uncertainty quantification in practical implementation (albeit with relatively strong assumptions regarding the data generating process). Our analysis focuses on bivariate relationships due to both computational tractability and fundamental identifiability constraints that become more severe in higher-dimensional settings. While this restriction limits immediate practical applicability, it allows us to develop clear theoretical insights that extend conceptually to more complex scenarios. We explicitly discuss the opportunities and challenges associated with more complex scenarios Section \ref{sec:extending_framework}.

\subsection{Motivating Example: Policy Evaluation Under Structural Uncertainty}
\label{sec:motivation}

The challenges of structural uncertainty in causal decision-making are perhaps best illustrated through a concrete policy evaluation scenario. Consider a government agency tasked with evaluating the effectiveness of a job training program using observational data on training participation ($X$) and subsequent employment outcomes ($Y$). The available data reveals a strong positive correlation between training participation and employment success, but two fundamentally different causal explanations for this correlation are plausible: Under the first causal structure ($\mathcal{G}_1: X \rightarrow Y$), training participation causally improves employment outcomes. This interpretation suggests that the program successfully enhances participants' skills, making them more attractive to employers. If this structure accurately describes reality, then expanding the program would lead to improved employment outcomes across the population. The policy implication hence is clear: increased investment in training programs represents an effective use of public resources.

The second plausible causal structure ($\mathcal{G}_2: Y \rightarrow X$) tells a different story. Here, better employment prospects lead individuals to seek training opportunities, perhaps because those with stronger job market fundamentals are more motivated to invest in skill development or because employers encourage training among workers they view as promising. Under this interpretation, the observed correlation reflects a selection effect rather than causation. Program expansion would have no effect on employment outcomes, making increased investment a misallocation of resources.

These competing causal narratives have drastically different policy implications despite building on the same correlational patterns in the data. Standard practice in policy evaluation would either take the causal direction as a given based on theoretical motivations or, more recently, proceed by using causal discovery methods (see below) to select one specific structure and then treating this structure as a known truth. Effectively, if a causal discovery algorithm suggests structure $\mathcal{G}_1$ with 60\% confidence, conventional approaches would proceed as if training definitely causes employment improvements, ignoring the substantial probability that the reverse relationship holds.

We propose and examine a conceptually simple model averaging approach to offer an alternative that explicitly accounts for structural uncertainty. Rather than committing to a single causal structure, we weight policy decisions by structural probabilities. If structure $\mathcal{G}_1$ receives 60\% posterior probability, we weight its implied optimal policy by 0.6 and weight structure $\mathcal{G}_2$'s implied policy by 0.4. This weighted approach naturally hedges against the possibility of structural misspecification while maintaining decision-theoretic optimality when uncertainty is genuine. The practical benefits of this approach depend on several key factors that our framework formalizes. When structural uncertainty is minimal—either because the data strongly support one structure or because domain knowledge provides clear guidance—the benefits of averaging diminish. Similarly, when different causal structures happen to imply similar optimal policies, structural uncertainty becomes inconsequential regardless of its magnitude. However, when structural uncertainty is substantial and different structures suggest markedly different actions, as in our training program example, the model averaging approach can yield significant improvements in decision quality.

The job training example also illustrates the broader methodological challenge we address: Causal discovery methods can provide evidence about causal structures, but they rarely quantify uncertainty in ways that facilitate principled decision-making. Our framework provides a start to bridge this gap by showing how to translate structural uncertainty into improved decisions, providing both theoretical foundations and practical implementation guidance.

In the remainder of this paper, we first briefly examine related work; we identify the different streams in the causal inference, causal discovery, and (Bayesian) decision making literature that our work builds upon. Next, in Section \ref{sec:framework}, we introduce our notation, theoretical framework, and we formalize our model averaging approach. In Section \ref{sec:when_matters}, given our formalization, we formalize the intuitions that dictate when model averaging over causal structures should impact decision making (as compared to simple, and common, causal structure selection). Section \ref{sec:optimality} finalizes our theoretical contributions by proving that Bayesian model averaging leads to both Bayes optimal and frequentist optimal decision making in a number of well-defined cases. In Section \ref{sec:empirical} we demonstrate that, in the bivariate case, recent bivariate discovery methods can be used to indeed improve practical decision making using the theoretical approach detailed in earlier sections. In Section \ref{sec:extending_framework} we discuss potential practical extension of our work to a multivariate setting. Finally, in Section \ref{sec:discussion} we explicitly discuss the shortcomings of our proposed method and the challenges that arise when extending these ideas to more complex settings.

\section{Related Work}
\label{sec:related}

Our work sits at the intersection of several established research areas, each contributing essential elements to our framework. We briefly relate our work to each related strand in the literature and highlight a number of closely related works at the end of this overview.

First, Bayesian decision theory provides the foundational framework for our approach to optimal decision-making under (structural) uncertainty. While classical treatments focus on parameter uncertainty --- as opposed to structural uncertainty --- within known model structures, the field has developed sophisticated methods for propagating uncertainty through to optimal decisions \citep{berger1985statistical, robert2007bayesian}. Recent advances in Bayesian causal inference have extended these principles to causal effect estimation, providing principled approaches to handling parameter uncertainty in treatment effect estimates \citep{hahn2020bayesian, hill2011bayesian}. However, these approaches universally assume that the causal structure itself is known, effectively (and often implicitly) treating the directed acyclic graph (DAG) representing causal relationships as fixed and known. Our work extends this literature by incorporating uncertainty about the causal structure itself into the decision-making framework.

The causal discovery literature has developed increasingly sophisticated methods for learning causal structures from observational data. Constraint-based approaches leverage conditional independence testing to infer causal relationships \citep{spirtes2000causation, zhang2008completeness}, while score-based methods search over causal structures using model selection criteria \citep{chickering2002optimal, heckerman1995learning}. More recent developments include methods specifically designed for bivariate causal discovery using assumptions about functional relationships and noise structures \citep{hoyer2009nonlinear, peters2014causal, janzing2012information}, some of which we return to in Section \ref{sec:empirical}. Despite these methodological advances, most causal discovery procedures provide point estimates of causal structures rather than principled uncertainty quantification. Notable exceptions include recent work on Bayesian causal discovery \citep{madigan1995bayesian, eaton2007exact}, but these methods have seen limited adoption in practice, partly due to computational challenges.

The Bayesian model averaging (BMA) literature demonstrates the theoretical and practical benefits of averaging over multiple plausible models for prediction and decision-making. Seminal work by \citet{hoeting1999bayesian} and \citet{raftery1997bayesian} established the theoretical foundations and practical implementation of BMA for prediction problems. Extensions to decision-making contexts show that averaging can provide substantial improvements when model uncertainty is genuine \citep{barbieri2004optimal, clyde2011bayesian}. However, applications to causal structures remain limited, with most work focusing on uncertainty about specific regression model specifications rather than causal relationships (although, when making causal decisions, often multiple regression models can be related to different adjustment sets implied by different underlying DAGs).

This paper contributes to a growing literature on robust causal inference which addresses various sources of uncertainty that standard approaches overlook. Methods for handling unmeasured confounding include sensitivity analysis \citep{rosenbaum2002observational, vanderweele2017sensitivity} and approaches that bound treatment effects under plausible assumptions about confounding strength \citep{masten2020salvaging, bonvini2022sensitivity}. Other work in this general area addresses uncertainty about functional form specifications through doubly robust methods \citep{bang2005doubly, kang2007demystifying} and machine learning approaches that provide robustness to model misspecification \citep{chernozhukov2018double, kennedy2020optimal}. While these methods all address important sources of uncertainty, they typically still assume that the causal structure itself is known (or if not, its uncertainty is only implicitly incorporated through the selection of different adjustment sets), complementing rather than substituting for our approach.

Recent developments have begun to bridge the gap between causal discovery and causal inference. \citet{sharma2021causal} developed methods for averaging over discovered causal graphs when estimating treatment effects and hence is very closely related to our proposals: we build on this work and extend it by both providing more formal insights into the scenarios in which causal structure uncertainty is of decision importance, and by providing practical examples in a bivariate case. Also strongly related, \citet{jung2021learning} propose approaches for learning causal structures under model uncertainty. However, the decision-theoretic foundations for structural averaging remain underdeveloped, and most work focuses on estimation rather than decision-making. Our contribution explicitly combines causal inference (i.e., causal effect estimation), causal discovery, and (Bayesian) decision-making.

The intersection of causal inference and decision theory has received increasing attention in recent years. \citet{manski2004social} pioneered the application of statistical decision theory to policy evaluation, emphasizing the importance of characterizing uncertainty for policy decisions. \citet{dehejia2005practical} and \citet{imbens2015causal} further developed these connections, showing how decision-theoretic principles can guide causal inference practice. However, this literature has again primarily focused on parameter uncertainty within known causal structures, leaving structural uncertainty largely unexplored. Our work fills this gap by providing a comprehensive framework for incorporating structural uncertainty into causal decision-making.

Finally, our approach also connects to recent work on causal effect estimation under model uncertainty. \citet{wang2012calibrating} develop methods for combining estimates from multiple causal identification strategies, while \citet{dorn2021doubly} propose doubly robust approaches that average over multiple model specifications. These methods address important practical concerns but focus on uncertainty about identification strategies rather than causal structures per se. Our framework provides a complementary approach that directly addresses structural uncertainty while maintaining decision-theoretic optimality.

\section{Notation and Problem Formalization}
\label{sec:framework}

This section establishes the formal framework for decision-making under structural uncertainty. We begin by defining our causal structures of interest and the effects they imply, then we introduce our decision-theoretic framework, and finally we formally present our model averaging approach alongside the conventional model selection alternative.

\subsection{Causal Structures and Identification}

We first provide an introduction to causal structures and causal effects with special interest in the bivariate case.

\subsubsection{Structural Causal Models}

We begin with the general framework of structural causal models (SCMs) \citep{pearl2009causality}. A structural causal model consists of three components: (1) a set of endogenous variables $\mathbf{V}$, (2) a set of exogenous variables $\mathbf{U}$, and (3) a set of structural equations that specify how each endogenous variable is determined by its direct causes. The causal relationships among endogenous variables are represented by a directed acyclic graph (DAG), where nodes represent variables and directed edges represent direct causal relationships. Formally, let $\mathcal{G} = (\mathbf{V}, \mathbf{E})$ denote a DAG where $\mathbf{V}$ is the set of vertices (variables) and $\mathbf{E}$ is the set of directed edges. Each variable $V_i \in \mathbf{V}$ is generated according to a structural equation of the form:
\begin{equation}
V_i = f_i(\text{PA}_i, U_i)
\end{equation}
where $\text{PA}_i$ denotes the parents of $V_i$ in the DAG (i.e., its direct causes), $U_i$ is an exogenous noise term, and $f_i$ is the structural function. The collection of exogenous variables $\mathbf{U} = \{U_1, U_2, \ldots\}$ are assumed to be jointly independent.

\subsubsection{Interventions and Causal Effects.}

Pearl's do-calculus \citep{pearl2009causality} provides a formal framework for defining causal effects through interventions within an SCM. An intervention $\text{do}(X = x)$ represents an external manipulation that sets variable $X$ to value $x$, effectively replacing the structural equation for $X$ with the constant assignment $X = x$. This intervention breaks incoming causal links to $X$ while preserving the causal mechanisms for all other variables. The causal effect of intervention $\text{do}(X = x)$ on outcome $Y$ (both in $\mathbf{V}$) is defined as:
\begin{equation}
\tau(x) = \mathbb{E}[Y \mid \text{do}(X = x)]
\label{eq:causal_effect}
\end{equation}
This quantity represents the expected value of $Y$ that would be observed if we were to intervene and set $X = x$ for all units in the population. This differs fundamentally from the conditional expectation $\mathbb{E}[Y \mid X = x]$, which represents the expected value of $Y$ among units that happen to have $X = x$ through natural processes rather than intervention. These quantities differ whenever there are unobserved or uncontrolled for common causes (confounders) of both $X$ and $Y$. For example, if an unobserved factor $U$ influences both $X$ and $Y$, then units with $X = x$ may systematically differ from the general population in ways that affect $Y$, making $\mathbb{E}[Y \mid X = x]$ a biased estimate of the causal effect. In contrast, $\mathbb{E}[Y \mid \text{do}(X = x)]$ represents the effect of setting $X = x$ while breaking any spurious associations induced by confounders.

\subsubsection{Bivariate Case}

Our analysis focuses on the bivariate case where the entire system consists of only two variables, $\mathbf{V} = \{X, Y\}$, and two competing causal structures are plausible:
\begin{align}
\mathcal{G}_1: \quad X &\rightarrow Y \\
\mathcal{G}_2: \quad Y &\rightarrow X
\end{align}
This restriction to bivariate relationships is motivated by both conceptual clarity --- the bivariate case suffices to discuss the advantages of our proposed model averaging approach (see below) --- and computational ease (computational tractability and identifiability issues become increasingly severe in higher-dimensional settings). While this limitation restricts immediate practical applicability, it allows us to develop clear theoretical insights that extend conceptually to more complex scenarios.

Each structure $\mathcal{G}_i$ (for $i \in \{1,2\}$) defines a structural causal model with parameters $\boldsymbol{\theta}_i$. Under structure $\mathcal{G}_1$, the data generating process follows:
\begin{align}
X &\sim P_X \\
Y &= f_1(X, U_Y)
\end{align}
where $U_Y \perp X$ represents unobserved factors affecting $Y$, and $f_1$ is the structural function mapping causes to effects. Similarly, but reversed, under structure $\mathcal{G}_2$, the causal relationships is:
\begin{align}
Y &\sim P_Y \\
X &= f_2(Y, U_X)
\end{align}
where $U_X \perp Y$ and $f_2$ is the corresponding structural function.

The causal effect of intervention $\text{do}(X = x)$ under structure $\mathcal{G}_i$ is defined as:
\begin{equation}
\tau_i(x) = \mathbb{E}[Y \mid \text{do}(X = x), \mathcal{G}_i, \boldsymbol{\theta}_i]
\label{eq:tau_i}
\end{equation}
This quantity is our primary interest in the remainder of this paper, as we will examine formal decision making using loss functions which include (finite data) estimates of $\tau_i(x)$ as one of their core arguments.

Of special interest in the literature is a quantity that is relevant specifically for binary treatments, $X \in \{0,1\}$. Here the average treatment effect (ATE) under structure $\mathcal{G}_i$ is:
\begin{equation}
\text{ATE}_i = \tau_i(1) - \tau_i(0)
\end{equation}

A crucial insight emerges from this formulation: under structure $\mathcal{G}_1$ defined above, variable $X$ causally affects $Y$, so interventions on $X$ produce meaningful effects on $Y$. Conversely, under structure $\mathcal{G}_2$, variable $X$ does not causally affect $Y$, implying that $\text{ATE}_2 = 0$ (or more generally, $\mathbb{E}[Y \mid \text{do}(X = x)] = \mathbb{E}[Y]$). This fundamental difference in causal effects between the two structures of interest forms the basis for our decision-theoretic analysis.

\subsection{Decision-Theoretic Framework}

We embed our causal analysis within a formal decision-theoretic framework: i.e., we are interested in understanding the properties of decisions which are dependent on (estimates of) causal effects. Let $\mathcal{A}$ denote the action space available to the decision-maker, and let $L(a, \mathcal{G}_i, \boldsymbol{\theta}_i)$ represent the loss incurred from taking action $a$ when the true causal structure is $\mathcal{G}_i$ with parameters $\boldsymbol{\theta}_i$. Throughout our analysis, we denote the true data generating process as $(\mathcal{G}^{\text{true}}, \boldsymbol{\theta}^{\text{true}})$, where $\mathcal{G}^{\text{true}} \in \{\mathcal{G}_1, \mathcal{G}_2\}$ represents the actual causal structure that generated the observed data, and $\boldsymbol{\theta}^{\text{true}}$ represents the true parameter values associated with this structure.

\begin{assumption}[Regularity Conditions]
\label{ass:regularity}
We are interested in decision problems which satisfy the following regularity conditions:
\begin{enumerate}
\item The action space $\mathcal{A}$ is compact.
\item The loss function $L(a, \mathcal{G}_i, \boldsymbol{\theta}_i)$ is jointly measurable and finite almost surely.
\item For each $(\mathcal{G}_i, \boldsymbol{\theta}_i)$, the loss function $L(\cdot, \mathcal{G}_i, \boldsymbol{\theta}_i)$ is continuous in $a$.
\end{enumerate}
\end{assumption}

These common regularity conditions ensure that optimal actions exist and that our theoretical results apply broadly to practical decision problems. To illustrate the framework, we present two canonical examples that capture common decision scenarios:

\begin{example}[Treatment Decision with Cost-Benefit Trade-off]
\label{ex:treatment_decision}
Consider a binary action $a \in \{0,1\}$ representing whether or not to treat a population (i.e., $X \in \{0,1\}$). A reasonable associated loss function is:
\begin{equation}
L(a, \mathcal{G}_i, \boldsymbol{\theta}_i) = c \cdot a - b \cdot \text{ATE}_i \cdot a + \eta_i \cdot a
\end{equation}
where $c > 0$ represents the cost of treatment, $b > 0$ weights the benefits from treatment effects, and $\eta_i$ captures any structure-specific costs or benefits (i.e., costs or benefits introduced explicitly due to the causal structure). Note that this example loss function --- often with $\eta_i = 0$ --- is very commonly used. Our theoretical results hold for this choice of loss function. However, we do not provide simulation results for this specific case as there are currently no well-functioning practical methods to estimate the probability of either causal structure $\mathcal{G}_i$, $i \in \{1,2\}$, for the case where $X$ is binary.
\end{example}

\begin{example}[Continuous Intervention Intensity]
\label{ex:continuous_intervention}
For a continuous action $a$ representing an intervention (e.g., a change in medication dose), the loss function for a given causal structure $\mathcal{G}_i$ and its parameters $\boldsymbol{\theta}_i$ can be generally defined as:
\begin{equation}
L(a, \mathcal{G}_i, \boldsymbol{\theta}_i) = \frac{1}{2}(\tau_i(x+a) - Y^*)^2 + \lambda a^2
\end{equation}
Here, $\tau_i(x+a) = \mathbb{E}[Y \mid \text{do}(X = x+a), \mathcal{G}_i, \boldsymbol{\theta}_i]$ is the expected outcome of $Y$ when $X$ is intervened upon to take the value $x+a$ under causal structure $\mathcal{G}_i$. $Y^*$ is a target outcome, and $\lambda \geq 0$ is a penalty coefficient for the magnitude of the intervention $a$.

In the simulation study later in this paper, we consider a specific scenario where the objective is to choose an action $a$ that aligns with the marginal causal effect $E_i(x) = \frac{d}{dx} \tau_i(x)$, while penalizing the action's magnitude. This corresponds to a simplified loss function:
\begin{equation}
L(a, \mathcal{G}_i, \boldsymbol{\theta}_i) = \frac{1}{2}(E_i(x) - a)^2 + \lambda a^2
\end{equation}
This form arises if, for instance, the target $Y^*$ is set such that the desired change in $Y$ is directly proportional to the marginal effect, and the action $a$ directly aims to achieve this change. The optimal action $a^*_{i}(x)$ for a given $x$ is found by minimizing this loss function with respect to $a$.
\end{example}
We explicitly investigate this second example in our simulation study below.

\subsection{Estimation of Structural Probabilities and Causal Effects}

To bridge the theoretical framework with practical applications, we need to estimate the necessary quantities from observed data. In practice, we observe data $\mathbf{D} = \{(x_j, y_j)\}_{j=1}^n$ drawn from the joint distribution implied by the true (but unknown) causal structure. Given a Bayesian approach and reasonable prior choice, from this data, we aim to compute two key posterior distributions:
\begin{enumerate}
\item Posterior structure probabilities: $P(\mathcal{G}_i \mid \mathbf{D})$ for $i \in \{1,2\}$.
\item Posterior parameter distributions: $P(\boldsymbol{\theta}_i \mid \mathcal{G}_i, \mathbf{D})$ for each structure.
\end{enumerate}
These posterior distributions incorporate both the information content of the data and any prior beliefs about structural plausibility. The structure probabilities $P(\mathcal{G}_i \mid \mathbf{D})$ quantify our uncertainty about which causal structure generated the observed data, while the parameter distributions $P(\boldsymbol{\theta}_i \mid \mathcal{G}_i, \mathbf{D})$ capture uncertainty about the strength of causal relationships within each structure. Note that in practice we do not have generic methods to compute $P(\mathcal{G}_i \mid \mathbf{D})$; in our theoretical analysis we assume this quantity given, while in our simulations we explore two different practical approximations.

\subsection{Decision Rules: Model Averaging vs. Model Selection}

Given posterior distributions over structures and parameters, we consider two distinct approaches to decision-making that represent fundamentally different strategies for handling structural uncertainty.

First we consider the common \emph{model selection} (MS) approach. The model selection approach first selects the most probable structure, then chooses the action that is optimal for that structure:
\begin{equation}
a_{MS} = \arg\min_{a \in \mathcal{A}} \mathbb{E}[L(a, \hat{\mathcal{G}}, \boldsymbol{\theta}_{\hat{\mathcal{G}}}) \mid \hat{\mathcal{G}}, \mathbf{D}]
\label{eq:model_selection}
\end{equation}
where $\hat{\mathcal{G}} = \arg\max_i P(\mathcal{G}_i \mid \mathbf{D})$ is the maximum a posteriori (MAP) structure. The model selection approach represents the conventional practice in applied causal inference, where practitioners typically commit to a single causal structure (either based on domain knowledge or causal discovery methods) and proceed as if this structure were known with certainty.

Alternatively, we consider a model averaging (MA) approach. The model averaging approach chooses the action that minimizes posterior expected loss while explicitly accounting for structural uncertainty:
\begin{equation}
a_{MA} = \arg\min_{a \in \mathcal{A}} \sum_{i=1}^2 \mathbb{E}[L(a, \mathcal{G}_i, \boldsymbol{\theta}_i) \mid \mathcal{G}_i, \mathbf{D}] \cdot P(\mathcal{G}_i \mid \mathbf{D})
\label{eq:model_averaging}
\end{equation}
This approach weights the expected loss under each structure by the posterior probability of that structure, thereby hedging against the possibility of structural misspecification.

\subsection{Notation Summary}

For reference, we summarize the key notation introduced in this section:

\begin{table}[h]
\centering
\begin{tabular}{| l | l |l}
\hline
\textbf{Symbol} & \textbf{Definition} \\
\hline
$\mathcal{G}_1, \mathcal{G}_2$ & Causal structures ($X \rightarrow Y$, $Y \rightarrow X$) \\
$\mathcal{G}^{\text{true}}, \boldsymbol{\theta}^{\text{true}}$ & True structure and parameters \\
$\boldsymbol{\theta}_i$ & Parameters under $\mathcal{G}_i$ \\
$\tau_i(x)$ & Causal effect of $\text{do}(X=x)$ on $Y$ under $\mathcal{G}_i$ \\
$\text{ATE}_i$ & Average treatment effect under $\mathcal{G}_i$ \\
$E_i(x)$ & Marginal causal effect of $X$ on $Y$ ($dY/dX$) under $\mathcal{G}_i$ \\
$\mathcal{A}$ & Action space \\
$L(a, \mathcal{G}_i, \boldsymbol{\theta}_i)$ & Loss function \\
$\mathbf{D}$ & Observed data \\
$P(\mathcal{G}_i \mid \mathbf{D})$ & Posterior over structure \\
$P(\boldsymbol{\theta}_i \mid \mathcal{G}_i, \mathbf{D})$ & Posterior over parameters \\
$a_{MA}, a_{MS}$ & Optimal actions under MA and MS strategies \\
$\hat{\mathcal{G}}$ & MAP structure \\
$a^*_i$ & Optimal action under $\mathcal{G}_i$ \\
\hline
\end{tabular}
\caption{Summary of notation used throughout.}
\end{table}

This framework provides the foundation for our subsequent analysis of when structural uncertainty matters for decision-making and under what conditions model averaging provides superior decisions compared to model selection. We start with a discussion of the cases in which structural uncertainty matters in decision making.

\section{When Does Structural Uncertainty Matter?}
\label{sec:when_matters}

Having established our formal framework, we now turn to one of the central questions of this paper: when does structural uncertainty have meaningful implications for decision-making, and under what conditions does model averaging provide substantial benefits over the conventional model selection approach? We introduce three factors.

\subsection{The Role of Structural Uncertainty Magnitude}

The first and most intuitive factor determining the importance of structural uncertainty is how uncertain we are about the true causal structure. This is quantified by the posterior probabilities $P(\mathcal{G}_i \mid \mathbf{D})$.

\paragraph{Proposition 1 (Extreme Certainty Cases).}
\label{prop:1}
\begin{enumerate}
\item If $P(\mathcal{G}_1 \mid \mathbf{D}) \to 1$, then $a_{MA} \to a_{MS}$ and both converge to the optimal action under $\mathcal{G}_1$.
\item If $P(\mathcal{G}_1 \mid \mathbf{D}) \to 0$, then $a_{MA} \to a_{MS}$ and both converge to the optimal action under $\mathcal{G}_2$.
\end{enumerate}
\noindent
Clearly, maximum divergence between $a_{MA}$ and $a_{MS}$ occurs when $P(\mathcal{G}_1 \mid \mathbf{D}) = P(\mathcal{G}_2 \mid \mathbf{D}) = 0.5$.

\subsection{The Importance of Effect Size Differences}

The second factor is how different the optimal actions are under different causal structures.
\paragraph{Proposition 2 (Effect Similarity and Decision Convergence).}
\label{prop:2}
If $|a^*_1 - a^*_2| \leq \epsilon$, then, regardless of structural uncertainty magnitude, $|a_{MA} - a_{MS}| \leq \epsilon$.

\subsection{Loss Function Sensitivity and Amplification Effects}

A third factor is how sensitive the loss function is to the estimated causal effect.
\paragraph{Definition 1 (Loss Function Sensitivity).}
A loss function $L(a, \mathcal{G}_i, \boldsymbol{\theta}_i)$ is $\kappa$-sensitive if for any $|a_1 - a_2| \geq \delta$:
\begin{equation}
|L(a_1, \mathcal{G}_i, \boldsymbol{\theta}_i) - L(a_2, \mathcal{G}_i, \boldsymbol{\theta}_i)| \geq \kappa \delta
\end{equation}


\paragraph{Proposition 3 (Sensitivity and Averaging Benefits).}
\label{prop:3}
Let $P_{\text{err}}$ denote the probability that model selection chooses the incorrect causal structure—that is, a structure $\mathcal{G}_i \neq \mathcal{G}^{\text{true}}$.  
Suppose the loss function is $\kappa$-sensitive, and let $\Delta := |a^*_1 - a^*_2| > 0$ be the difference between the optimal actions under $\mathcal{G}_1$ and $\mathcal{G}_2$. Then the expected reduction in loss from model averaging (MA) compared to model selection (MS) satisfies:
\begin{align}
&\mathbb{E}[L(a_{MS}, \mathcal{G}^{\text{true}}, \boldsymbol{\theta}^{\text{true}})]
 - \mathbb{E}[L(a_{MA}, \mathcal{G}^{\text{true}}, \boldsymbol{\theta}^{\text{true}})] \nonumber \\
&\quad \geq \kappa \cdot \Delta \cdot \min\left\{P(\mathcal{G}_1 \mid \mathbf{D}),\; P(\mathcal{G}_2 \mid \mathbf{D})\right\} \cdot P_{\text{err}}
\end{align}

\noindent
This bound shows that the benefit of model averaging increases with three key factors:  
\begin{enumerate}
\item the action gap $\Delta$ between the optimal policies under each structure,  
\item the sensitivity $\kappa$ of the loss function to action deviations, and  
\item the chance of model selection error $P_{\text{err}}$, which is amplified when the posterior over structures is not sharply peaked.
\end{enumerate}

\section{Optimality of Model Averaging}
\label{sec:optimality}

In this section, building on the framework and assumptions introduced above, we establish that a Bayesian model averaging approach to structural uncertainty yields optimal decisions under a wide range of conditions. We present both Bayesian and frequentist optimality results, and discuss situations where model averaging may be suboptimal due to violations of key assumptions.

\subsection{Main Theoretical Results}

We begin by formalizing the assumptions under which our optimality results hold.

\begin{assumption}[Well-Specified Hierarchical Model]
\label{ass:hierarchical}
We assume that the data-generating process (DGP) follows a hierarchical model contained within the candidate model space:
\begin{enumerate}
    \item \textbf{Structure selection:} $\mathcal{G}^{\text{true}} \sim \text{Bernoulli}(\pi)$, where $\mathcal{G}^{\text{true}} \in \{\mathcal{G}_1, \mathcal{G}_2\}$.
    \item \textbf{Parameter generation:} $\boldsymbol{\theta}^{\text{true}} \mid \mathcal{G}^{\text{true}} \sim P(\boldsymbol{\theta} \mid \mathcal{G}^{\text{true}})$.
    \item \textbf{Data generation:} $\mathbf{D} \sim P(\mathbf{D} \mid \mathcal{G}^{\text{true}}, \boldsymbol{\theta}^{\text{true}})$.
\end{enumerate}
We further assume that the analyst's prior beliefs match the true generative process: $P(\mathcal{G}_1) = \pi$ and $P(\boldsymbol{\theta} \mid \mathcal{G}_i) = P(\boldsymbol{\theta} \mid \mathcal{G}^{\text{true}} = \mathcal{G}_i)$. Thus, prior misspecification does not influence the results.
\end{assumption}

\begin{theorem}[Bayesian Optimality of Model Averaging]
\label{thm:bayes_optimal}
Under Assumptions~\ref{ass:regularity} and~\ref{ass:hierarchical}, the action $a_{MA}$ defined in Equation~\eqref{eq:model_averaging} minimizes the posterior expected loss:
\begin{equation}
    a_{MA} = \arg\min_{a \in \mathcal{A}} \mathbb{E}[L(a, \mathcal{G}^{\text{true}}, \boldsymbol{\theta}^{\text{true}}) \mid \mathbf{D}].
\end{equation}
\end{theorem}


\begin{proof}
Under Assumption~\ref{ass:regularity}, the loss function $L(a, \mathcal{G}, \boldsymbol{\theta})$ is measurable and bounded below, ensuring that a minimizing action exists. 

By the law of total expectation and the hierarchical Bayesian model (Assumption~\ref{ass:hierarchical}), the posterior expected loss given data $\mathbf{D}$ is:
\begin{equation}
\mathbb{E}[L(a, \mathcal{G}^{\text{true}}, \boldsymbol{\theta}^{\text{true}}) \mid \mathbf{D}] 
= \sum_i \mathbb{E}[L(a, \mathcal{G}_i, \boldsymbol{\theta}_i) \mid \mathcal{G}_i, \mathbf{D}] \cdot P(\mathcal{G}_i \mid \mathbf{D}).
\end{equation}
The action $a_{MA}$ defined in Equation~\eqref{eq:model_averaging} is the minimizer of this expression by construction. Hence, it minimizes the posterior expected loss.
\end{proof}

\begin{theorem}[Frequentist Optimality of Model Averaging]
\label{thm:freq_optimal}
Under Assumptions~\ref{ass:regularity} and~\ref{ass:hierarchical}, the decision rule $\delta(\mathbf{D}) = a_{MA}$ minimizes the frequentist risk:
\begin{equation}
    R(\delta) = \mathbb{E}_{\mathcal{G}^{\text{true}}, \boldsymbol{\theta}^{\text{true}}, \mathbf{D}}[L(\delta(\mathbf{D}), \mathcal{G}^{\text{true}}, \boldsymbol{\theta}^{\text{true}})]
\end{equation}
over the class of all measurable decision rules.
\end{theorem}


\begin{proof}
We evaluate the frequentist risk of a decision rule $\delta(\mathbf{D})$ under the data-generating process implied by the hierarchical model (Assumption~\ref{ass:hierarchical}):
\begin{equation}
R(\delta) = \mathbb{E}_{\mathcal{G}^{\text{true}}, \boldsymbol{\theta}^{\text{true}}, \mathbf{D}}[L(\delta(\mathbf{D}), \mathcal{G}^{\text{true}}, \boldsymbol{\theta}^{\text{true}})].
\end{equation}
Applying the law of total expectation:
\begin{align}
R(\delta) 
&= \mathbb{E}_{\mathbf{D}}\left[ \mathbb{E}_{\mathcal{G}^{\text{true}}, \boldsymbol{\theta}^{\text{true}} \mid \mathbf{D}}[L(\delta(\mathbf{D}), \mathcal{G}^{\text{true}}, \boldsymbol{\theta}^{\text{true}})] \right].
\end{align}
For each realization of $\mathbf{D}$, the inner expectation is minimized by the action $a_{MA}$ that minimizes posterior expected loss (by Theorem~\ref{thm:bayes_optimal}). Since this holds pointwise for all $\mathbf{D}$, it follows that $\delta(\mathbf{D}) = a_{MA}$ minimizes the total frequentist risk $R(\delta)$ over all measurable decision rules.
\end{proof}

The above results demonstrate the theoretical superiority of including structural uncertainty into decision-making that depends on causal effect estimates. Combined with the previous section which outlines when differences between a model averaging and a model selection approach are expected to be large, these results should provide justification for practical attempts to include structural uncertainty into decision making in various applied problems.

\subsection{Conditions Under Which Optimality Fails}

While model averaging is optimal under the idealized assumptions discussed above, practical limitations may lead to suboptimal outcomes and encourage different decisions. 

\begin{proposition}[Suboptimality Under Model Misspecification]
\label{prop:misspecification}
If the true data-generating process does not follow --- or is not properly approximated by --- the hierarchical structure of Assumption~\ref{ass:hierarchical}, model averaging may yield suboptimal decisions. In particular, if the posterior structure probabilities $P(\mathcal{G}_i \mid \mathbf{D})$ are poorly calibrated, a model selection approach may outperform model averaging.
\end{proposition}
\noindent
We provide a stylized example of when this could be the case:

\begin{example}[Failure Under Extreme Loss Functions]
\label{ex:extreme_loss}
Consider a loss function with hard thresholds:
\begin{equation*}
    L(a, \mathcal{G}_i, \boldsymbol{\theta}_i) =
    \begin{cases}
        0 & \text{if } |\text{ATE}_i \cdot a| > \tau, \\
        M & \text{otherwise},
    \end{cases}
\end{equation*}
where $M \gg 0$ and $\tau > 0$. If the competing structures imply effects of opposite signs and $P(\mathcal{G}_1 \mid \mathbf{D}) \approx 0.5$, model averaging may yield $a \approx 0$, incurring loss $M$ under both structures. In contrast, model selection commits to one structure and may avoid loss entirely.
\end{example}

\section{Empirical Performance}
\label{sec:empirical}

In this section we detail the design and results of a simulation study, intended to empirically demonstrate the benefits of model averaging. Given that no generic method exists to properly quantify $P(\mathcal{G} \mid \mathbf{D})$, we focus a) on the continuous case and hence the loss function described in Example \ref{ex:continuous_intervention}, and b) we explicitly generate data using non-linearity and heteroskedasticity; exactly the assumptions that can be exploited by causal discovery methods to "orient the causal arrow". 

\subsection{Data Generating Processes (DGPs)}
\label{subsec:dgp}

We consider bivariate data $(X, Y)$ generated from two distinct causal structures: $X \to Y$ (denoted as $\mathcal{G}_1$) and $Y \to X$ (denoted as $\mathcal{G}_2$). The true DGP in each simulation run is conceptually simple: first, a causal structure ($X \to Y$ or $Y \to X$) is selected with probability $\frac{1}{2}$. Next, the following structural causal models are used to generate both nonlinear and heteroskedastic datasets and to vary the magnitude of the causal effect:

In the cases when $X$ causes $Y$, $X$ is drawn from a standard normal distribution, $X \sim \mathcal{N}(0, 1)$. The relationship for $Y$ is then defined based on $X$:
\begin{itemize}
    \item \textbf{Heteroskedastic DGP:} This DGP featured a linear mean relationship but with noise variance dependent on the cause variable, implementing heteroskedasticity:
    \begin{equation*}
    Y = \beta X + \epsilon_Y, \quad \epsilon_Y \sim \mathcal{N}(0, 0.5 + 0.3 |X| )
    \end{equation*}
    \item \textbf{Nonlinear DGP:} This DGP included a quadratic term, implementing a non-linear relationship:
    \begin{equation*}
    Y = \beta X + \gamma X^2 + \epsilon_Y, \quad \epsilon_Y \sim \mathcal{N}(0, 0.3)
    \end{equation*}
\end{itemize}
For both DGPs, two effect sizes (small and large) were implemented:
\begin{itemize}
    \item \textbf{Small Effect Size:} For heteroskedastic, $\beta = 0.5$. For nonlinear, $\beta = 0.5, \gamma = 0.1$.
    \item \textbf{Large Effect Size:} For heteroskedastic, $\beta = 2$. For nonlinear, $\beta = 2, \gamma = 0.5$.
\end{itemize}
The true \emph{marginal} causal effect, $E(x) = dY/dX = \frac{d}{dx} \tau_1(x)$, is derived from these functions. For the heteroskedastic DGP, $E(x) = \beta$ (a constant marginal effect). For the nonlinear DGP, $E(x) = \beta + 2\gamma X$. To relate this marginal causal effect to the ATE introduced in Section \ref{sec:framework}, note that for a given DGP with a continuous treatment, the ATE can be computed by averaging $E(x)$ over the distribution of $X$, i.e., $\mathbb{E}[E(X)]$. If $X$ is a binary variable, the ATE would typically be defined as $\tau(1) - \tau(0)$. This binary ATE corresponds to the integral of the marginal causal effect over the range of the intervention, i.e., $\int_0^1 E(x) dx$. For example, if $E(x)$ is a constant effect (e.g., $E(x) = \beta$), then $\tau(1) - \tau(0) = \int_0^1 \beta dx = \beta$.

In the cases where the true causal arrow is $Y \to X$, we simply use the same data generating mechanism as above with $X$ and $Y$ inverted. However, note that under this causal structure ($\mathcal{G}_2$), the true causal effect of $X$ on $Y$, $E(x)$, is $0$ for all $x$, and consequently, the ATE is also $0$.

\subsection{Causal Discovery Methods}
Two causal discovery methods were employed to infer the direction of causality between $X$ and $Y$: Additive Noise Models (ANM) \citep{peters2014causal} and a Regression method \citep[see, e.g.,][]{zhang2017causal}. In both cases we utilized a bootstrapping approach to compute $P(\mathcal{G} \mid \mathbf{D})$ \citep{efron1994introduction}.

\subsubsection{Additive Noise Models (ANM)}
The ANM-based discovery method leverages the principle that if $X \to Y$ is the true causal direction, then $Y = f(X) + \epsilon_Y$ where $X$ and $\epsilon_Y$ are independent. Conversely, if $Y \to X$ is true, then $X = g(Y) + \epsilon_X$ where $Y$ and $\epsilon_X$ are independent. We implemented the following procedure:
\begin{itemize}
    \item For each direction ($X \to Y$ and $Y \to X$), a Generalized Additive Model (GAM) is fitted (e.g., $Y \sim s(X)$ for $X \to Y$). We used the `gam` function from the `mgcv` package in R for fitting the GAMs.
    \item The residuals from the GAM are then tested for independence with the respective cause variable using both the standard Pearson correlation and a distance correlation (\texttt{energy::dcor} in \texttt{R}). A lower combined score (sum of absolute correlation and distance correlation) indicates stronger independence.
    \item For each bootstrap sample $m$, $m = 100$ in our simulations, the above ANM test was performed. The proportion of bootstrap samples favoring $X \to Y$ (i.e., having a lower score for $Y \sim s(X)$) is taken as the probability $P(\mathcal{G}_1 \mid \mathbf{D})$.
\end{itemize}

\subsubsection{Regression Method (Regression)}
This method assesses the fit and independence of residuals from polynomial regressions in both directions.
\begin{itemize}
    \item For each direction, a second-degree polynomial regression was fitted (e.g., $Y \sim \text{poly}(X, 2)$ for $X \to Y$). We chose a second-degree polynomial to capture potential non-linear relationships, aligning with the nonlinear DGP used in our simulations.
    \item The $R^2$ value of the regression model and the Kendall's $\tau$ correlation between the residuals and the cause variable were computed.
    \item A combined score is calculated as $R^2 - |\text{Kendall's }\tau|$, aiming for high fit and low residual-cause dependence.
    \item Similar to ANM, the probability $P(\mathcal{G}_1 \mid \mathbf{D})$ was computed as the proportion of bootstrap samples where the combined score for $X \to Y$ is higher than for $Y \to X$.
\end{itemize}

\subsection{Causal Effect Estimation}
\label{sec:causal_effect_est}

In each simulation run, once a causal direction is estimated (i.e., $P(\mathcal{G} \mid \mathbf{D})$ is computed as described above), the causal effect of $X$ on $Y$, $E(x) = dY/dX$, was estimated given the data. We implement the following for the two cases $\mathcal{G}_1$ and $\mathcal{G}_2$:
\begin{itemize}
    \item If $\mathcal{G}_1$ (i.e., $X \to Y$):
    \begin{itemize}
        \item For heteroskedastic DGPs, a linear model ($Y \sim X$) is fitted, and the coefficient of $X$, $\beta_x$, is taken as the effect.
        \item For nonlinear DGPs, a GAM ($Y \sim s(X)$) is fitted. The marginal effect $E(x)$ is then numerically approximated by computing the derivative of the fitted GAM at each point $X$. This was achieved using finite differences.
    \end{itemize}
    \item If $\mathcal{G}_2$: The direct causal effect of $X$ on $Y$ is assumed to be $0$ in all cases.
\end{itemize}
The output of this step in the simulations study is a function representing the estimated marginal causal effect of $X$ on $Y$ (or $0$ if $Y \to X$ is assumed), which is used later on to compute the selected action in each simulation run under the two decision strategies of interest. We introduce notation $\hat{E}(x)$ for the estimated causal effect given the data and the selected structure (in the heteroskedastic case $\hat{E}(x) = \beta_x$). Note that we effectively use point-estimates for $\hat{E}(x)$ given the simulation data and thus do not implement a fully Bayesian setup for causal effect estimation within each structure.

\subsection{Decision Making and Loss Computation}
This section details how actions are chosen based on the causal discovery results and the two decision strategies of interest, Model Selection (MS) or Model Averaging (MA), as well as the computation of the loss. In all simulations, we used the following loss function (a special case of Example~\ref{ex:continuous_intervention}) for each simulated individual $i = 1, \dots, n$:
\begin{equation*}
L_i = 0.5 \left(a'(x_i) - a^*(x_i)\right)^2 + \lambda a'(x_i)^2,
\end{equation*}
where $a'(x_i)$ is the action determined by the decision strategy under consideration, and $\lambda \geq 0$ represents the \emph{cost of intervention}. The first term penalizes deviation from the true optimal action $a^*(x_i)$, while the second term discourages overly aggressive actions by making large interventions costly.

For a given data point $x_i$, the true optimal action $a^*(x_i)$ minimizes the loss function under the true causal structure and parameters. As shown in Example~\ref{ex:continuous_intervention} earlier, this is given by:
\begin{equation*}
a^*(x_i) = \frac{E_{\text{true}}(x_i)}{1 + 2\lambda},
\end{equation*}
when the causal direction is $X \to Y$, and $a^*(x_i) = 0$ otherwise. Here, $E_{\text{true}}(x_i)$ is the true marginal causal effect of $X$ on $Y$ at $x_i$ (as discussed in Section~\ref{subsec:dgp}).

\paragraph{Derivation of Optimal Action $a^*(x_i)$:}
The optimal action $a^*(x_i)$ is derived by minimizing the loss function $L(a) = \frac{1}{2}(E_{true}(x_i) - a)^2 + \lambda a^2$ with respect to $a$. To find the minimum, we take the derivative of $L(a)$ with respect to $a$ and set it to zero:

$$ \frac{dL}{da} = \frac{1}{2} \cdot 2 (E_{true}(x_i) - a) \cdot (-1) + 2\lambda a = 0 $$
$$ -(E_{true}(x_i) - a) + 2\lambda a = 0 $$
$$ -E_{true}(x_i) + a + 2\lambda a = 0 $$
$$ a(1 + 2\lambda) = E_{true}(x_i) $$

Solving for $a$ yields the optimal action:
$$ a^*(x_i) = \frac{E_{true}(x_i)}{1 + 2\lambda} $$

The selected action, $a'(x_i)$, given a decision strategy of interest (MS or MA), is computed by substituting the estimated causal effect $\hat{E}(x)$ (its computation described in Section \ref{sec:causal_effect_est} above) for $E(x)$, and given this estimated causal effect, the resulting estimated optimal action is:
\[
\hat{a}^*(x_i) = \frac{\hat{E}(x_i)}{1 + 2\lambda}
\]
To compute $a'(x_i)$ we use:
\begin{itemize}
    \item \textbf{Model Selection (MS):}\\
    If the discovered probability $P(\mathcal{G}_1 \mid \mathbf{D})$ is greater than or equal to $0.5$, the decision-maker selects the $X \to Y$ causal graph ($\mathcal{G}_1$) as the most probable structure. The chosen action $a'_{MS}(x_i)$ is then the optimal action derived assuming $\mathcal{G}_1$ is the true graph. Otherwise, if $P(\mathcal{G}_1 \mid \mathbf{D})$ is less than $0.5$, the $Y \to X$ causal graph ($\mathcal{G}_2$) is selected, and the chosen action $a'_{MS}(x_i)$ is the optimal action derived assuming $\mathcal{G}_2$ is the true graph. Thus:
    \[
    a'_{MS}(x_i) =
    \begin{cases}
        \frac{\hat{E}(x_i)}{1 + 2\lambda} & \text{if } P(X \to Y) \ge 0.5 \\
        0 & \text{if } P(X \to Y) < 0.5
    \end{cases}
    \]

    \item \textbf{Model Averaging (MA):}\\
    The decision-maker computes a weighted average of the optimal actions derived from both hypothesized causal graphs, $\mathcal{G}_1$ and $\mathcal{G}_2$. The weights are the respective discovered probabilities, $P(X \to Y)$ and $P(Y \to X) = 1 - P(X \to Y)$. Thus, the chosen action $a'_{MA}(x_i)$ is:
    \[
    a'_{MA}(x_i) = P(X \to Y) \cdot \frac{\hat{E}(x_i)}{1 + 2\lambda} + (1 - P(X \to Y)) \cdot 0
    \]
\end{itemize}

Given both $a'(x_i)$ and $a^*(x_i)$ the individual level loss is easily computed. The overall loss for a given simulation run is the mean of these individual losses across all $n$ data points. In our reporting below we focus on the difference between the two decision strategies of interest, $\Delta L = L_{MS} - L_{MA}$.

\subsection{Summarizing our setup}
\label{subsec:study_summary}

To summarize, the simulation study systematically varies several parameters to cover a wide range of scenarios. For each unique combination of parameters, $N_{reps} = 100$ replications were performed. The parameters varied are:

\begin{table}[h]
\centering
\begin{tabular}{| l | l |}
\hline
\textbf{Parameter} & \textbf{Values} \\
\hline
Sample Size ($n$) & $10, 50, 100, 500$ \\
DGP Type(s) & `heteroskedastic`, `nonlinear` \\
Effect Size & `large`, `small` \\
Cost of Intervention ($\lambda$) & $0.1, 0.9$ \\
Causal Discovery Method & `ANM`, `Regression` \\
\hline
\end{tabular}
\caption{Simulation Study Parameters}
\end{table}

We compute $\Delta L$ for each scenario, and we separate the two causal discovery methods in our presentation of the obtained results. Note that our parameters relate directly to our theoretical results:
\begin{enumerate}
\item An average $\Delta L > 0$ over all simulations implies generic superiority of MA over MS, empirically supporting Theorems \ref{thm:bayes_optimal} and \ref{thm:freq_optimal}.
\item The sample sizes directly relate to the ability of the causal discovery method to determine $P(\mathcal{G}_1 \mid \mathbf{D})$; thus, aligning with Proposition \ref{prop:1}, we should find that $\Delta L$ is largest for small sample sizes (as in these cases there is more uncertainty regarding the causal direction).
\item The effect sizes relate directly and in a straightforward way to Proposition \ref{prop:2}, where we should see $\Delta L$ being larger for large effect sizes.
\item The intervention costs, represented by $\lambda$, also play an important role in shaping the benefits of model averaging. The loss function becomes more sensitive to deviations from the optimal action as $\lambda$ increases, since its curvature is given by $\kappa = 1 + 2\lambda$. However, higher intervention costs also discourage large actions: the optimal action shrinks as $\lambda$ increases, following $a^*(x_i) = E_{\text{true}}(x_i) / (1 + 2\lambda)$. As a result, both the model selection (MS) and model averaging (MA) strategies tend to recommend smaller actions when $\lambda$ is high. Because these actions are closer to zero, the difference between them—and thus the difference in incurred loss, $\Delta L$—also becomes smaller. This trade-off between increased sensitivity and dampened actions leads to a net decrease in the benefit of model averaging at higher intervention costs. 
\end{enumerate}

\subsection{Results}

We present the results of our simulation study in terms of the average loss difference $\Delta L = L_{MS} - L_{MA}$ between model selection (MS) and model averaging (MA) strategies. A positive value of $\Delta L$ indicates superior performance of model averaging over model selection, providing empirical validation of our theoretical framework.

\subsubsection{Overall Performance Comparison}

Our simulation results provide strong empirical support for the theoretical superiority of model averaging over model selection. Averaged across all 6,400 simulation runs, we find $\Delta L = 0.103$ with a highly significant test statistic ($t = 29.156$, $p < 0.001$ for $H_0: \Delta L = 0$). This substantial positive difference confirms that model averaging consistently outperforms model selection across the range of scenarios considered, directly supporting Theorems \ref{thm:bayes_optimal} and \ref{thm:freq_optimal}.

The performance advantage varies considerably between the two causal discovery methods employed. As shown in Table \ref{tab:discovery-method-stats}, the Additive Noise Models (ANM) method exhibits more pronounced differences between strategies, with $\Delta L = 0.144$ compared to $\Delta L = 0.0618$ for the Bootstrap Regression method. Both differences are highly statistically significant ($p < 0.001$). The larger performance gap observed with ANM likely reflects its greater sensitivity to the specific data generating processes used in our simulations, which were explicitly designed to exploit the non-linearity and heteroskedasticity assumptions that ANM leverages for causal orientation.

Figure \ref{fig:delta_loss} illustrates the distribution of $\Delta L$ values across all simulation runs, stratified by discovery method. The figure clearly demonstrates that while both methods benefit from model averaging, ANM shows both higher mean improvements and greater variability in performance gains. This pattern suggests that the choice of causal discovery method significantly influences the magnitude of benefits achievable through model averaging, with methods that produce more uncertain or variable probability estimates showing greater improvements from averaging.

\begin{table}[htpb!]
\centering
\label{tab:discovery-method-stats}
\begin{tabular}{|l|r|c|c|c|}
\hline
\textbf{Discovery Method} & \textbf{n} & \textbf{Mean $\Delta L$} & \textbf{SD $\Delta L$} & \textbf{SE $\Delta L$} \\
\hline
ANM & 3200 & 0.144 & 0.325 & 0.00574 \\
Regression & 3200 & 0.0618 & 0.226 & 0.00399 \\
\hline
\end{tabular}
\caption{Summary Statistics by Discovery Method}
\end{table}

\begin{figure}[htpb!]
    \centering
    \includegraphics[width=0.8\textwidth]{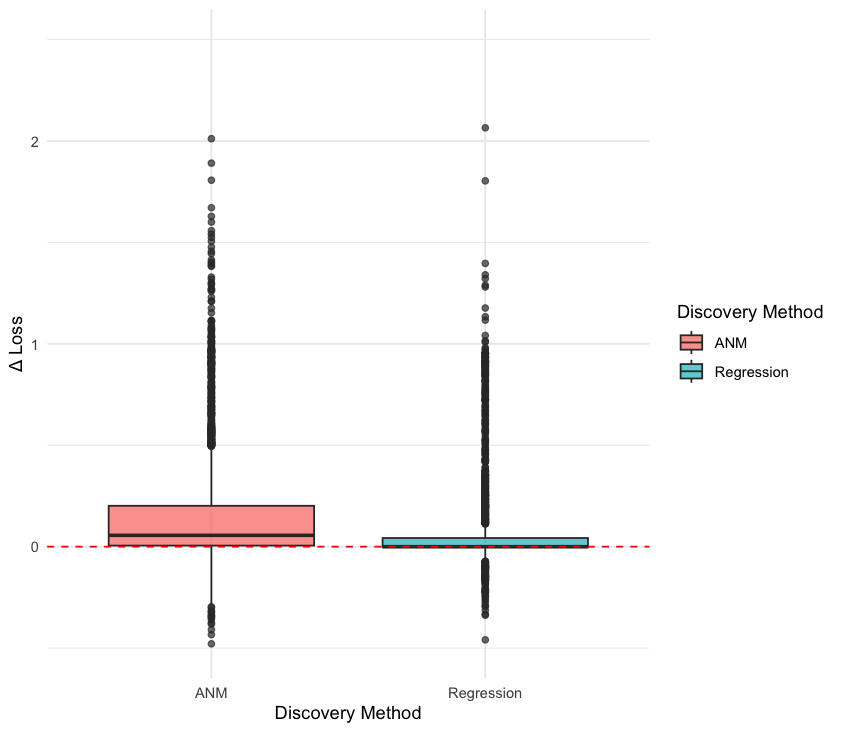}
    \caption{Distribution of loss differences ($\Delta L$) between model selection and model averaging strategies, stratified by causal discovery method. Positive values indicate superior performance of model averaging.}
    \label{fig:delta_loss}
\end{figure}

\subsubsection{Sample Size Effects}

The relationship between sample size and the benefits of model averaging provides direct empirical validation of Proposition \ref{prop:1}. Figure \ref{fig:prop1_n} demonstrates a clear negative relationship between sample size ($n$) and $\Delta L$, confirming our theoretical prediction that model averaging offers the greatest advantages when causal discovery is most uncertain.

As sample sizes increase, the probability estimates $P(\mathcal{G}_1 \mid \mathbf{D})$ produced by both causal discovery methods become increasingly concentrated near the boundaries (0 or 1), reflecting growing confidence in the identified causal direction. This convergence reduces the effective difference between model selection and model averaging strategies, as the averaging weights become increasingly skewed toward a single model. Conversely, at smaller sample sizes where $P(\mathcal{G}_1 \mid \mathbf{D})$ remains closer to intermediate values, model averaging provides substantial hedging benefits against model selection errors.

A simple linear regression of $\Delta L$ on sample size yields a statistically significant negative relationship ($\beta_1 = -0.0001$, $p < 0.001$), quantifying the systematic decrease in model averaging benefits as sample information accumulates. This finding has important practical implications, suggesting that model averaging strategies are particularly valuable in small-sample settings where causal inference is inherently more uncertain.

\begin{figure}[htpb!]
    \centering
    \includegraphics[width=0.8\textwidth]{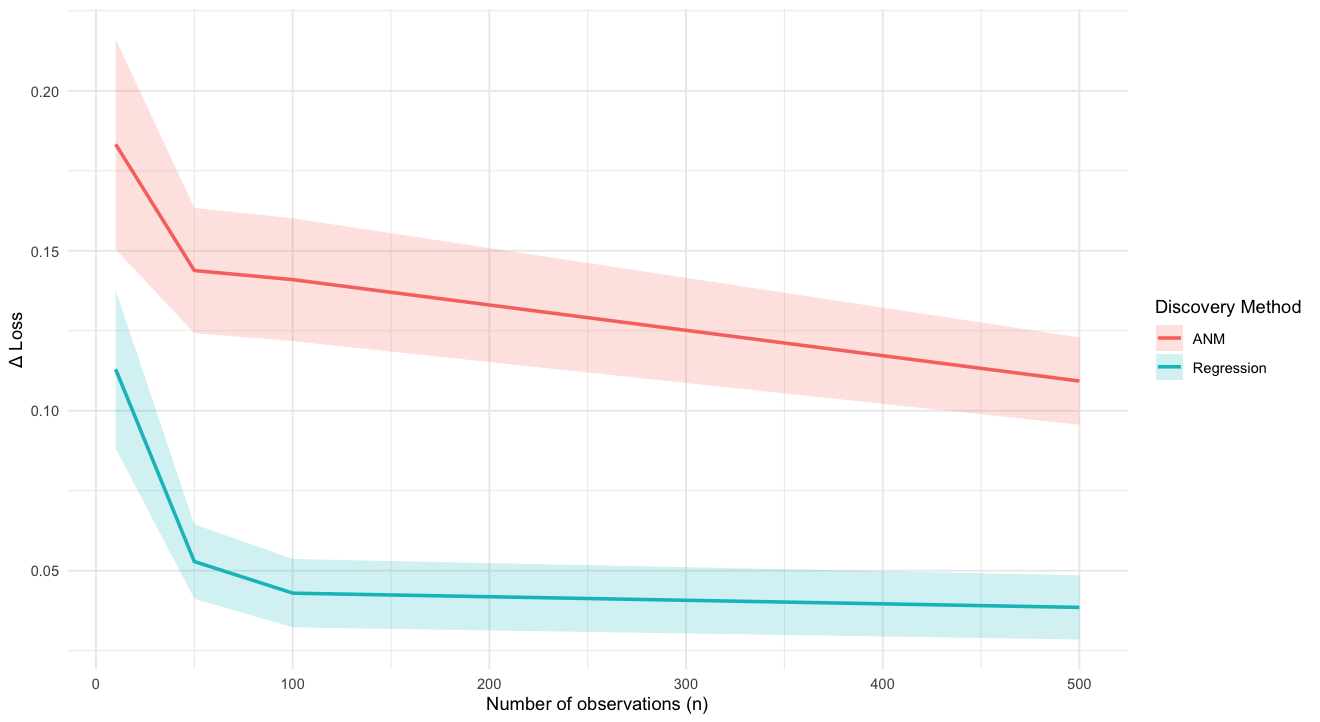}
    \caption{Relationship between sample size and loss difference ($\Delta L$) between model selection and model averaging. The negative trend confirms that model averaging benefits are greatest when sample sizes are small and causal discovery uncertainty is high.}
    \label{fig:prop1_n}
\end{figure}

\subsubsection{Effect Size Dependencies}

The relationship between true causal effect magnitudes and model averaging performance provides mixed support for Proposition \ref{prop:2}. Table \ref{tab:summary-stats} reveals that the effect of causal strength varies substantially between discovery methods, highlighting the complex interplay between data generating mechanisms and causal inference procedures.

For the ANM method, we observe only modest differences between large and small effect sizes ($\Delta L = 0.151$ vs. $0.137$ respectively), with both showing substantial benefits from model averaging. This relatively stable performance across effect sizes suggests that ANM's reliance on independence testing maintains consistent sensitivity to model uncertainty regardless of causal strength.

In contrast, the Bootstrap Regression method exhibits a pronounced effect size dependency, with large effects yielding substantially greater model averaging benefits ($\Delta L = 0.100$) compared to small effects ($\Delta L = 0.0233$). This pattern aligns more closely with Proposition \ref{prop:2}, as stronger causal relationships provide clearer signals that polynomial regression can exploit for directional inference, leading to more reliable probability estimates and, paradoxically, greater benefits when those estimates are averaged rather than selected categorically.

Figure \ref{fig:prop2_eff} illustrates these differential patterns across discovery methods and effect sizes. The interaction between method choice and effect magnitude underscores the importance of considering both the data generating process and the discovery algorithm when evaluating model averaging strategies.

\begin{table}[htpb!]
\centering
\label{tab:summary-stats}
\begin{tabular}{|l|l|r|r|r|r|}
\hline
\textbf{Effect Size} & \textbf{Discovery Method} & \textbf{n} & \textbf{Mean $\Delta L$} & \textbf{SD $\Delta L$} & \textbf{SE $\Delta L$} \\
\hline
Large & ANM & 1600 & 0.151 & 0.315 & 0.00788 \\
Large & Regression & 1600 & 0.100 & 0.247 & 0.00618 \\
Small & ANM & 1600 & 0.137 & 0.335 & 0.00836 \\
Small & Regression & 1600 & 0.0233 & 0.195 & 0.00488 \\
\hline
\end{tabular}
\caption{Summary Statistics by Effect Size and Discovery Method}
\end{table}

\begin{figure}[htpb!]
    \centering
    \includegraphics[width=0.8\textwidth]{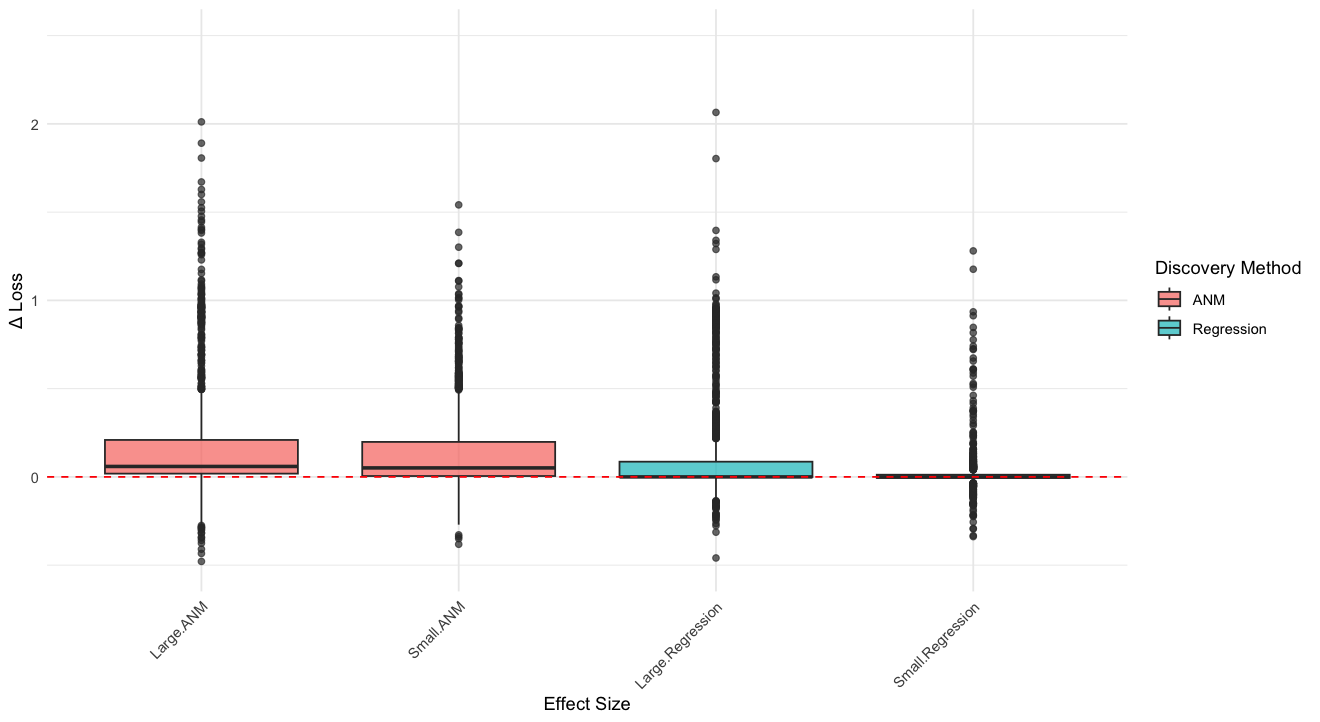}
    \caption{Model averaging benefits ($\Delta L$) by effect size and discovery method. The Bootstrap Regression method shows greater sensitivity to effect magnitude than ANM.}
    \label{fig:prop2_eff}
\end{figure}

\subsubsection{Intervention Cost Effects}

The influence of intervention costs ($\lambda$) on model averaging performance reveals the economic trade-offs inherent in causal decision-making under uncertainty. As anticipated from our simulations setup and choice of loss function, higher intervention costs reduce the absolute benefits of model averaging. Table \ref{tab:lambda-discovery-stats} demonstrates this pattern clearly: for both discovery methods, $\Delta L$ decreases  as $\lambda$ increases from 0.1 to 0.9. Under low intervention costs ($\lambda = 0.1$), ANM yields $\Delta L = 0.186$ while Bootstrap Regression produces $\Delta L = 0.0828$. These benefits are approximately halved under high intervention costs ($\lambda = 0.9$), with corresponding values of $0.103$ and $0.0409$ as expected given the mechansim described in the Section \ref{subsec:study_summary}.


\begin{table}[htpb!]
\centering
\label{tab:lambda-discovery-stats}
\begin{tabular}{|c|l|r|c|c|c|}
\hline
\textbf{$\lambda$} & \textbf{Discovery Method} & \textbf{n} & \textbf{Mean $\Delta L$} & \textbf{SD $\Delta L$} & \textbf{SE $\Delta L$} \\
\hline
0.1 & ANM & 1600 & 0.186 & 0.370 & 0.00926 \\
0.1 & Regression & 1600 & 0.0828 & 0.302 & 0.00755 \\
0.9 & ANM & 1600 & 0.103 & 0.266 & 0.00664 \\
0.9 & Regression & 1600 & 0.0409 & 0.101 & 0.00251 \\
\hline
\end{tabular}
\caption{Summary Statistics by Lambda and Discovery Method}
\end{table}

Figure \ref{fig:prop3_dep} visualizes the relationships across the parameter space, showing that while intervention costs moderate the absolute benefits of model averaging, they do not eliminate its fundamental advantages. This finding has important implications for practical applications, suggesting that model averaging strategies provide value across a range of economic environments, with the greatest absolute benefits realized in settings where intervention costs are relatively low.

\begin{figure}[htpb!]
    \centering
    \includegraphics[width=0.8\textwidth]{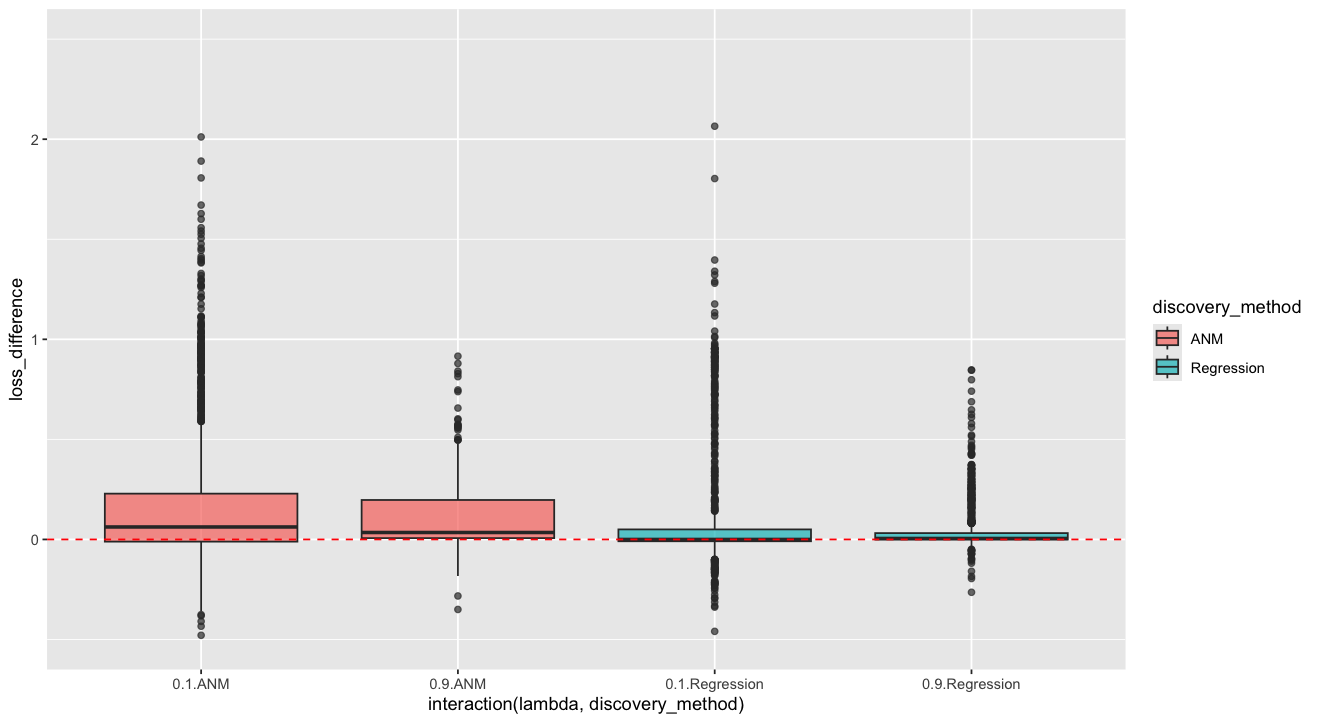}
    \caption{Impact of intervention costs ($\lambda$) on model averaging benefits ($\Delta L$) across discovery methods. Higher intervention costs reduce absolute benefits but preserve relative advantages of model averaging.}
    \label{fig:prop3_dep}
\end{figure}

\section{Extending our framework}
\label{sec:extending_framework}

Our analysis has established a theoretical foundation and provided empirical evidence for the benefits of Bayesian model averaging over causal structures, particularly in the bivariate setting. However, extending this framework to more complex, real-world scenarios requires addressing several key challenges. This section outlines the primary limitations of our current framework and proposes a research agenda for its practical extension to a generic set of variables with an unknown causal structure and a given loss function.

\subsection{Generic Computation of $P(\mathcal{G}|D)$}

A cornerstone of our model averaging framework is the posterior probability of a causal structure given the data, $P(\mathcal{G}|D)$. While our simulations in the bivariate case leveraged specific causal discovery methods (ANM and Regression based) combined with bootstrapping to approximate these probabilities, a general, robust, and computationally efficient method for computing $P(\mathcal{G}|D)$ remains a significant open challenge.

In the bivariate setting, even for simple $X \leftrightarrow Y$ relationships, accurately quantifying $P(\mathcal{G}|D)$ is non-trivial. Methods like ANM rely on strong assumptions about functional forms and noise distributions, which may not hold in practice. Furthermore, the bootstrapping approach, while providing a heuristic for uncertainty, is not a direct Bayesian posterior. Future work therefore needs to develop more principled Bayesian causal discovery methods that directly output posterior probabilities over a set of candidate DAGs. This could involve developing more robust and assumption-lean methods for bivariate causal discovery that can directly provide Bayesian posteriors over $X \rightarrow Y$ vs. $Y \rightarrow X$, potentially through non-parametric Bayesian approaches or methods that explicitly model (latent) confounders. 

For multivariate settings, extending the principles of constraint-based causal discovery (e.g., PC, FCI algorithms) to a Bayesian context could be promising. Instead of yielding a single Markov equivalence class, these methods could provide a posterior distribution over such classes, or directly over DAGs within those classes, by quantifying uncertainty in conditional independence tests. This would necessitate developing Bayesian hypothesis tests for conditional independence that yield posterior probabilities rather than binary decisions. Moreover, in many practical scenarios, strong domain expertise exists. Incorporating subjective prior beliefs about causal relationships (e.g., certain edges are highly improbable, or certain variables are known causes/effects) could significantly constrain the search space and improve the calibration of $P(\mathcal{G}|D)$. This necessitates developing elicitation methods for structural priors and integrating them seamlessly into Bayesian causal discovery algorithms. Such methods would be particularly valuable when observational data is sparse or noisy, allowing domain knowledge to compensate for data limitations. 

The development of methods that provide well-calibrated $P(\mathcal{G}|D)$ values, especially for arbitrary sets of variables, is perhaps the most critical research direction for the practical adoption of causal model averaging.

\subsection{Computational Complexity}

The computational complexity of Bayesian model averaging over causal structures grows rapidly with the number of variables. The number of possible DAGs on $p$ variables is super-exponential in $p$, making exhaustive enumeration and averaging computationally prohibitive for even moderately sized systems. This combinatorial explosion necessitates strategies to manage the search space and inference. A primary approach to mitigate this complexity is to intelligently prune the set of candidate DAGs. This can be achieved by focusing the search on relevant parent sets of the intervention variable and the outcome variable if the decision problem involves interventions on a specific variable or a small set of variables. Strong prior beliefs can also be used to assign zero probability to certain DAGs, effectively eliminating them from the search space, though this requires careful elicitation of expert knowledge. Alternatively, employing score-based or constraint-based heuristics to identify a small set of "most probable" or "high-scoring" DAGs, and then performing model averaging only over this restricted set, introduces a trade-off between computational tractability and theoretical optimality (as the true DAG might be excluded), but may be a necessary practical compromise. 

Beyond restricting the search space, developing efficient averaging algorithms that can perform the averaging without explicitly enumerating all DAGs is crucial. This could involve Markov Chain Monte Carlo (MCMC) methods that sample from the posterior distribution of DAGs, or variational inference approaches that approximate this posterior. For very large systems, a modular approach to causal discovery might be feasible, decomposing the problem into smaller, more manageable sub-problems, learning local causal structures around variables of interest, and then combining these local insights for decision-making. Addressing computational complexity is paramount for moving beyond toy examples and applying causal model averaging to real-world, high-dimensional decision problems.

\subsection{Robustness and Practical Applicability}

Our theoretical results demonstrate that model averaging is beneficial under specific conditions, particularly when structural uncertainty is moderate, causal effects differ substantially between structures, and loss functions are sensitive to these differences. However, the practical application of these methods requires careful consideration of various factors influencing their robustness and utility, forming a crucial part of the research agenda. First, the quality of the model averaging strategy critically depends on the accuracy and calibration of the estimated $P(\mathcal{G}|D)$. If causal discovery methods provide poorly calibrated probabilities (e.g., overconfident estimates), model averaging might perform worse than a robust model selection strategy. Research is therefore needed to develop diagnostic tools for assessing the calibration of structural uncertainty estimates and to create methods that yield more reliable uncertainty quantification. Second, while our framework is most relevant for observational data where structural uncertainty is inherent, even in experimental settings (e.g., randomized controlled trials), structural uncertainty might arise for other variables or when combining experimental and observational data. Future work should explore hybrid approaches that leverage experimental insights while accounting for residual structural uncertainty. 

Practically, decision-makers need guidance on how to assess these properties for their specific loss functions. This involves developing tools to perform sensitivity analysis on the loss function with respect to variations in causal effects, helping decision-makers understand when structural uncertainty is likely to have a significant impact. Furthermore, for various common decision problems (e.g., policy interventions, medical treatments), characterizing the typical range of $\Delta$ (difference in optimal actions) between competing causal structures would provide practical benchmarks for assessing the potential benefits of model averaging. Ultimately, the goal is to develop a comprehensive framework that allows practitioners to quantify and incorporate structural uncertainty into decision-making for any potential set of variables, with a given loss function, and an unknown causal structure. 

\section{Discussion}
\label{sec:discussion}

This paper has explored the critical, yet often overlooked, role of structural uncertainty in causal decision-making, particularly focusing on bivariate relationships. We have demonstrated both theoretically and empirically that explicitly accounting for uncertainty about the underlying causal graph through Bayesian model averaging can lead to superior decision quality compared to conventional model selection approaches. Our theoretical framework precisely characterizes the conditions under which structural uncertainty becomes decision-relevant: when uncertainty is moderate, when optimal actions differ substantially across competing structures, and when loss functions are sufficiently sensitive to these differences. 

Our contributions are multi-faceted. We have formalized a decision-theoretic framework for incorporating structural uncertainty, establishing the Bayesian and frequentist optimality of model averaging under well-specified hierarchical models. Through extensive simulations using both heteroskedastic and nonlinear data generating processes, we have empirically validated these theoretical predictions, showing consistent benefits of model averaging, especially in small-sample settings and when causal effects are pronounced. Furthermore, we demonstrated how modern bivariate causal discovery methods, despite their inherent assumptions, can be practically leveraged to quantify structural uncertainty and improve decision outcomes. 

While our analysis provides significant insights, it also highlights substantial avenues for future research, as detailed in Section \ref{sec:extending_framework}. The primary challenge lies in developing generic, robust, and computationally efficient methods for quantifying $P(\mathcal{G}|D)$ in multivariate settings. Additionally, addressing the super-exponential growth in the number of candidate DAGs will necessitate intelligent pruning strategies and advanced averaging algorithms. Ultimately, the goal is to develop a comprehensive and practically applicable framework for incorporating structural uncertainty into causal decision-making across diverse domains, ensuring that decisions are robust to the fundamental ambiguity of causal structure.

\section{Author statements}

\textbf{Author Contributions:} The ideas, manuscript, derivations, and simulations have all been carried out by the first author. During the preparation of this work the author(s) used several commercially available LLMs (including ChatGPT, Gemini, and Claude) in order to refine writing clarity, explore related literature, and verify proof concepts. After using these tools, the author(s) reviewed and edited the content as needed and take(s) full responsibility for the content of the published article.

\textbf{Conflict of interest statement:} There are no conflicts of interest to declare.

\textbf{Funding information:} No explicit funding source was used for this work / the work is not tied to a specific grant.

\textbf{Data availability statement:} No empirical data was used in this paper. The [R] code for the simulations will be made available on GitHub.

\textbf{Acknowledgments:} We would like to thank all the members of the TRI-DSA group at the University of Eindhoven for their encouragements and insightful comments.

\bibliographystyle{unsrtnat}
\bibliography{ref_cleaned}

\end{document}